\documentclass[runningheads]{llncs}
\usepackage{graphicx}
\usepackage{algorithm}
\usepackage{algorithmic}
\usepackage{amsmath}
\usepackage{amsfonts}
\usepackage{siunitx}
\usepackage{caption}
\usepackage{subcaption}

\usepackage{xr}
\makeatletter
\newcommand*{\addFileDependency}[1]{
\typeout{(#1)}
\@addtofilelist{#1}
\IfFileExists{#1}{}{\typeout{No file #1.}}
}\makeatother

\newcommand*{\myexternaldocument}[1]{%
\externaldocument[app:]{#1}%
\addFileDependency{#1.tex}%
\addFileDependency{#1.aux}%
}
\myexternaldocument{supplement}

\begin{document}
\title{lil’HDoC: An Algorithm For Good Arm Identification Under Small Threshold Gap}
%
\titlerunning{Lil’HDoC}
%
\author{Tzu-Hsien Tsai \and
Yun-Da Tsai \and
Shou-De Lin}
\authorrunning{Tsai et al.}
%
\institute{National Taiwan University\\
\email{alan070516180339887@gmail.com}\\
\email{\{f08946007,sdlin\}@csie.ntu.edu.tw}}

\maketitle              

\begin{abstract}
Good arm identification (GAI) is a pure-exploration bandit problem in which a single learner outputs an arm as soon as it is identified as a good arm. A good arm is defined as an arm with an expected reward greater than or equal to a given threshold. This paper focuses on the GAI problem under a small threshold gap, which refers to the distance between the expected rewards of arms and the given threshold.
We propose a new algorithm called lil'HDoC to significantly improve the total sample complexity of the HDoC algorithm. We demonstrate that the sample complexity of the first $\lambda$ output arm in lil'HDoC is bounded by the original HDoC algorithm, except for one negligible term, when the distance between the expected reward and threshold is small.
Extensive experiments confirm that our algorithm outperforms the state-of-the-art algorithms in both synthetic and real-world datasets.

\end{abstract}
\section{Introduction}

The stochastic multi-armed bandit problem (MAB) is a well-known task with various applications. In this problem, there are $K$ arms in the environment, and in each round, denoted by $t$, the learner selects an action $a_t \in [K]$ based on a policy and then pulls the corresponding arm. The selected arm generates an independent and identically distributed (i.i.d) reward $X_{a_t}(t)$ from an unknown distribution $v_i$ with an unknown expectation $u_i$. Subsequently, the learner observes the exact reward and updates its policy to achieve a specific objective. The classical objective of MAB is to minimize cumulative regret, where the learner aims to maximize the cumulative reward over a fixed number of trials \cite{auer2002finite}. In this setting, the learner faces the exploration-exploitation dilemma, where exploration involves pulling seemingly sub-optimal arms to discover the arm with the highest expected reward, while exploitation involves selecting the arm with the highest empirical reward to increase the cumulative reward.

One well-studied variant of MAB is best arm identification (BAI), a pure exploration problem where the learner aims to find the arm with the highest expected reward with as few samples as possible \cite{kaufmann2016complexity,kalyanakrishnan2012pac}. In 2016, Locatelli et al. \cite{locatelli2016optimal} proposed the threshold bandit problem (TBP), a specific instance of the pure exploration bandit framework \cite{chen2014combinatorial}. TBP divides all arms into two groups based on whether their expected reward is above a given threshold or not, using minimal samples. However, in some real-world applications, neither identifying the best arm (BAI) nor correctly partitioning all the arms (TBP) is necessary. Instead, it is more desirable to quickly identify a set of reasonably good arms. For example, in recommendation tasks, the goal is not necessarily to identify the most popular item or all popular items, but rather to quickly identify a set of items that are popular enough to recommend. To address this, Kano et al. \cite{kano2019good} proposed the good arm identification (GAI) framework. GAI has the same goal as TBP, but it aims to minimize the sample complexity not only of identifying all good arms but also of identifying the first $\lambda$ good arms, where $\lambda \in \mathbb{N}$. GAI faces a new type of exploration-exploitation dilemma, the exploration-exploitation dilemma of confidence, where exploration means the learner samples suboptimal arms to identify a good arm, and exploitation means the learner continues to sample the currently best arm to increase the confidence in its goodness.

HDoC \cite{kano2019good} has been proposed as the SOTA method to solve the GAI task.
However, the sample complexity shown in Table \ref{tbl:hdoccomplexity} suggests that HDoC can be quite expensive to use when the threshold gap is small. This is because it requires a large number of samples on a single arm in order to determine whether it is a good or bad arm. For example, consider a recommendation system that uses bandit to determine which items to recommend, with rewards associated with click-through-ratio (CTR). In such systems, many items may have very small CTR (e.g. close to zero), while the good arms (e.g. CTR much larger than zero) are sparse. In this scenario, $\Delta$ is likely to be small, which can significantly hurt the GAI performance. Therefore, we propose to decrease the confidence bound in the identification method of HDoC.

In this paper, we consider the GAI framework with i.i.d Bernoulli reward when the threshold gap is small ($<$ 0.01).
Inspired by the challenging situation, we propose a new algorithm called "lil'HDoC" where "lil" stands for the Law of Iterated Logarithm. In the bandit problem, the operation of sampling (or pulling the arm) is considered as acquiring information about the sampled arm. HDoC suggests that at the beginning of the algorithm, each arm should be sampled once. However, if the threshold gap is small, sampling each arm only once might not be sufficient to make the right decision in the subsequent sampling algorithm because we have less confidence in the goodness/badness of each arm.
One change we propose in our algorithm is to sample each arm more than once in the beginning, denoted as $T>1$. We will show that the value of $T$ can be determined based on the acceptance error rate and the number of arms. Sampling each arm for $T$ times in the beginning can harm the sample complexity when it is easy to identify the goodness/badness of arms. However, in a challenging situation where a large number of samples is needed to identify one arm, the effect of sampling each arm $T$ times is negligible.
By sampling each arm more than once in the beginning of the algorithm, we have more confidence in the goodness/badness of arms, and therefore can obtain a tighter confidence bound than HDoC. However, $T$ cannot be too large, otherwise, it can hurt the overall performance. Thus, determining a suitable value for $T$ while still maintaining the theoretical performance guarantee is the main challenge we address. We need to adjust the confidence bound for identification to reach the theoretical guarantee.

Our contribution is as follow:
\begin{itemize}
    \item Applying the law of iterative logarithm to design a new algorithm, named lil'HDoC that improves the total sample complexity of the HDoC algorithm in the context of Good Arm Identification (GAI).
    \item Exhibiting a PAC bounded sample complexity, particularly when the distance between the expected reward and the threshold is small.
    \item Providing various experiments to show that the lil'HDoC algorithm surpasses state-of-the-art algorithms in both synthetic and real-world datasets.
\end{itemize}

This paper is organized as follow.
Section \ref{sec:relatedworks} reviews the related works on GAI and threshold bandit problem.
Section \ref{sec:setting} and \ref{sec:preliminary} provide the basics about GAI.
Our algorithm and the proof of the theoretical guarantee is provided in section \ref{sec:algorithm}.
Section \ref{sec:experiment} describes the experiment results before the conclusion in section \ref{sec:conclusion}.

\section{Background} \label{sec:relatedworks}
\subsection{Good Arm Identification}
Kano~\cite{kano2019good} proposed a formulation of the good arm identification problem, derived from the threshold bandit problem. This formulation addresses a new type of exploration-exploitation dilemma, called the dilemma of confidence. In this dilemma, exploration refers to the agent pulling other arms, rather than the currently best one, to increase confidence in identifying whether it is good or bad. Exploitation involves pulling the currently best arm to increase confidence in its goodness. The algorithm is decomposed into two parts: the sampling method and the identifying method. The former is responsible for selecting the arm to sample in each round, and the latter decides whether an arm is good or bad. For the sampling method, Kano et al. proposed a Hybrid algorithm for the Dilemma of Confidence (HDoC) and two baseline algorithms: the Lower and Upper Confidence Bounds algorithm for GAI (LUCB-G), based on the LUCB algorithm for best arm identification~\cite{kalyanakrishnan2012pac}, and the Anytime Parameter-free Thresholding algorithm for GAI (APT-G), based on the APT algorithm for the thresholding bandit problem~\cite{locatelli2016optimal}.
The lower bound on the sample complexity for GAI is $\Omega(\lambda log \frac{1}{\delta})$, and HDoC can find $\lambda$ good arms within $O(\frac{\lambda \log \frac{1}{\delta} + (K - \lambda) \log \log \frac{1}{\delta} + K \log \frac{K}{\Delta}}{\Delta^2})$ samples.

\begin{table}[h]
    \centering
    \begin{tabular}{ |c||c|}
    \hline
         & HDoC\\\hline\\[-10pt]
        First $\lambda$ arm & $O \big( \dfrac{\lambda \log \frac{1}{\delta} + (K - \lambda) \log \log \frac{1}{\delta} + K \log \frac{K}{\Delta}}{\Delta^2} \big)$\\ [10pt]
        Total & $O\big( \dfrac{K \log \frac{1}{\delta} + K \log K + K \log \frac{1}{\Delta}}{\Delta^2} \big)$\\ [5pt]
    \hline
    \end{tabular}
\caption{Sample Complexity of HDoC}
\label{tbl:hdoccomplexity}
\end{table}

\section{Problem Setting} \label{sec:setting}
Let $K$ denote the number of arms, $\xi$ the threshold, and $\delta$ the acceptable error rate. The reward of each arm $i \in [1, \dots, K]$ follows a Bernoulli distribution with mean $\mu_i$, which is unknown to the learner. We define "good" arms as those whose means are larger than the threshold $\xi$. Without loss of generality, we can assume that the means of the arms are ordered such that:

\begin{equation}
\mu_1 \geq \mu_2 \geq \dots \geq \mu_{m} \geq \xi \geq \mu_{m + 1} \geq \dots \geq \mu_K
\end{equation}

Note that the learner is unaware of the number of "good" arms and their indexing.

In each round $t$, the agent selects an arm $a(t)$ to pull and receives a reward that is i.i.d. generated from distribution $v_{a(t)}$. The agent identifies one arm as the good arm and the rest as bad arms based on the rewards received from them in previous rounds. The agent stops at time $\tau_{stop}$ when all the good arms are identified. The objective is to minimize the upper bound of $\tau_{\lambda}$ (i.e., the number of sample times to identify $\lambda$ good arms) and $\tau_{stop}$ (i.e., the number of samples required to identify all good arms) with an acceptance error rate of $\delta$.

\section{Preliminary} \label{sec:preliminary}

\begin{table}[t]
    \centering
    \begin{tabular}{ |c||c|}
    \hline
        $K$ & Number of arms.\\
        $A$ & Set of arms, where $|A|=K$.\\
        $\xi$ & Threshold determining whether an arm is good or not.\\
        $\delta$ & Acceptance error rate. $\delta \leq 1/e$ \\
        $\mu_i$ & The true mean of $i^{th}$ arm.\\
        $\hat{\mu}_{i, t}$ & The empirical mean of $i^{th}$ arm at time $t$.\\
        $\tau_{\lambda}$ & Round that learner identifies $\lambda$ good arms.\\
        $\tau_{stop}$ & Round that learner identifies every arms.\\
        $N_i(t)$ & The number of samples of arm $i$ by the end of round $t$.\\
        $T_i(t)$ & The number of times of arm $i$ be pulled at $t$.\\
        $\overline{\mu}_{i,t}$ & Upper confidence bound of arm $i$ at time $t$.\\
        $\underline{\mu}_{i,t}$ & Lower confidence bound of arm $i$ at time $t$.\\
        $\Delta_i$ & $\big| \mu_i - \xi \big|$\\
        $\Delta_{i, j}$ & $\mu_i - \mu_j$\\
        $\Delta$ & $\min(\min_{i \in K} \Delta_i, \min_{j \in [K - 1]} \frac{\Delta_{j, j + 1}}{2})$\\
    \hline
    \end{tabular}
\caption{Notation list}
\label{tbl:notation}
\end{table}

The notation is listed in Table \ref{tbl:notation}.
Two important lemmas are stated below. The first is the finite form of the Law of Iterated Logarithm \cite{jamieson2014lil} which is the kernel of our new confidence bound. The second is the inequality involving the operation of iterated logarithm.
\begin{lemma} [Finite form of LIL]
    \label{lma:iteratedbound}
    Let $X_1, X_2, \dots X_n$ be i.i.d. $\sigma$--sub--gaussian random variables. Then for algorithm parameters $\epsilon \in (0, 1)$ and $\rho \in (0, \dfrac{\log{(1 + \epsilon)}}{e})$, with probability at least $1 - c_{\epsilon} \rho^{1 + \epsilon}$,
    \begin{equation}
        \frac{1}{t} \sum_{s = 1}^{t} X_s \leq U(t, \rho)
    \end{equation}
    for all $t > 0$. Here, $U$ is the upper confidence bound
    \begin{equation}
        U(t, \omega) = (1 + \sqrt{\epsilon}) \sqrt{\frac{2 \sigma^2 (1 + \epsilon)}{t} \log [\frac{\log ((1 + \epsilon)t)}{\omega}]}
    \end{equation}
    \begin{equation}
        c_{\epsilon} = \dfrac{2 + \epsilon}{\epsilon} \bigg[ \frac{1}{\log{(1 + \epsilon)}} \bigg]^{1 + \epsilon}
    \end{equation}
\end{lemma}
\begin{proof}
    See Section 4, Lemma 1 in \cite{jamieson2014lil}.
\end{proof}

We will use $U(t, \omega)$ and $c_\epsilon$ throughout the paper for brevity.
It is worth noting that the Bernoulli distribution is $\frac{1}{2}$-sub-gaussian.

\begin{lemma} \cite{jiang2017practical}
    \label{lma:iteratedlogtransform}
    For $t \geq 1$, $\epsilon \in (0, 1)$, $c > 0$, \text{ and} $0 < \omega < 1$,
    \begin{equation}
        \frac{1}{t} \log{\bigg[\dfrac{\log{[(1 + \epsilon)t]}}{\omega} \bigg]} \geq c \Rightarrow t \leq \frac{1}{c} \log{\bigg[ \dfrac{2\log{\big[\dfrac{(1 + \epsilon)}{c\:\omega}\big]}}{\omega} \bigg]}
    \end{equation}
\end{lemma}
\begin{proof}
    Direct calculation. For details, see \cite{jiang2017practical}.
\end{proof}

\renewcommand{\algorithmicrequire}{\textbf{Input:}}
\renewcommand{\algorithmicensure}{\textbf{Initialize:}}

\begin{algorithm}[htb]
\caption{lil'HDoC}
\label{alg:lilhdoc}
\begin{algorithmic}[1]
    \REQUIRE K, $\delta$ ($\leq 1/e$)
    \STATE $T_i(K) = 1,\: \forall \: i \in [K]$
    \STATE $A = [K]$
    \STATE $B = K + 1$
    \STATE $C = \max(\frac{1}{\delta}, e)$
    \STATE $r = (1 + \sqrt{\epsilon})^2(1 + \epsilon)$
    \STATE $\epsilon = \text{argmax}_{\epsilon \in [0, 1]} \big(r - 1 \leq \min \big( \frac{\log \log B}{\log B}, \frac{\log \log C}{\log C} \big)\big)$
    \STATE $T = \text{argmin}_{t > 0} \big( \frac{t^2}{[\log [(1 + \epsilon)t] ]^{r}} \geq \frac{1}{4} K^{r - 1} (\frac{1}{\delta})^{r - 1} (c_\epsilon)^{r} \big)$
    \FOR {$s \in A$}
        \STATE Pull arm $s$ $T$ times
        \STATE $T_i(s) = T$
    \ENDFOR
    
    \WHILE{$A \neq \phi$}
        \FOR {$s \in A$}
            \STATE $u_{i, T_i(t)} \gets \sqrt{\frac{\log t}{2 N_i(t)}}$
        \ENDFOR
        \STATE $h_t = \text{argmax}_{i \in [K]} \big(\hat{\mu}_{i, T_i(t)} + u_{i, T_i(t)} \big)$
        \STATE Pull $h_t$
        \IF{$\hat{\mu}_{i, T_i(t)} - U(N_{h_t}(t), \frac{\delta}{c_\epsilon K}) \geq \xi$}
            \STATE Output $h_t$ as a good arm
            \STATE Delete $h_t$ from $A$
        \ELSIF{$\hat{\mu}_{i, T_i(t)} + U(N_{h_t}(t), \frac{\delta}{c_\epsilon K}) \leq \xi$}
            \STATE Delete $h_t$ from $A$
        \ENDIF
    \ENDWHILE
\end{algorithmic}
\end{algorithm}

\begin{table}[t]
    \centering
    \begin{tabular}{ |c||c|}
    \hline
         & lil'HDoC\\\hline\\[-10pt]
        First $\lambda$ arm & $O \big( K \log (K + 1) \log [\max (\frac{1}{\delta}, e)] + $\newline$\dfrac{\lambda \log \frac{1}{\delta} + (K - \lambda) \log \log \frac{1}{\delta} + K \log \frac{K}{\Delta}}{\Delta^2}$\\[15pt]
        Total & $O\big( \dfrac{K \log \frac{1}{\delta} + K \log K + K \log \log \frac{1}{\Delta}}{\Delta^2} \big)$\newline \smallskip + $O\big(K \log(K + 1) \log (\max(\frac{1}{\delta}, e))\big)$\\[6pt]
    \hline
    \end{tabular}
\caption{Sample Complexity of lilHDoC}
\label{tbl:lilhdoccomplexity}
\end{table}

\section{Algorithm} \label{sec:algorithm}
We propose a new algorithm, lil'HDoC. This algorithm improve the complexity of HDoC by sampling every arms more than one time in the beginning and redesigning the confidence bound in the identifying method.
The algorithm is provided in Algorithm \ref{alg:lilhdoc}, and its sample complexity is listed in Table \ref{tbl:lilhdoccomplexity}.
In lil'HDoC, the confidence bound in the identification method has a faster convergence rate than HDoC algorithm, thanks to the term related to $N_i(t)$ being in the form of $\sqrt{\frac{\log \log N_i(t)}{N_i(t)}}$, instead of $\sqrt{\frac{\log N_i(t)}{N_i(t)}}$. Consequently, the required number of samples decreases.
It can be observed that for any $c_1,: c_2 \in \mathbb{R}^+$, there exists a value $T$ such that for all $t > T$,
\begin{equation}
c_1 \sqrt{\frac{\log t}{t}} \geq c_2 \sqrt{\frac{\log \log t}{t}}
\end{equation}
This implies that after sampling arm $i$ $T$ times, lil'HDoC can identify it as well as HDoC.

Moreover, the explicit form of $T$ is easy to handle, and we apply binary search to determine its value in lil'HDoC. We can consider sampling each arm for $T$ times as a way to gain sufficient confidence in the goodness or badness of the arms, which enables us to develop a more precise confidence bound for identifying the arms in lil'HDoC. In the complexity analysis of the first $\lambda$ arm in Table \ref{tbl:lilhdoccomplexity}, the first term corresponds to $T$, while the second term corresponds to the sample complexity of HDoC. In challenging situations, the terms divided by $\Delta$ dominate the sample complexity, making the $K \log (K + 1) \log [\max(\frac{1}{\delta}, e)]$ negligible. Thus, the sample complexity of the first $\lambda$ arms is the same under big-O notation in challenging situations. Additionally, when an arm is sampled more than $T$ times, lil'HDoC requires fewer samples to identify that arm. While the total complexity in Table \ref{tbl:lilhdoccomplexity} could potentially be $O(K\log(K + 1)\log(\max(\frac{1}{\delta}, e)))$, we will demonstrate in section \ref{sec:totalcomplexity} that this scenario can be easily avoided. We will now establish the theoretical guarantee for the lil'HDoC algorithm by proving the correctness of the algorithm, providing the sample complexity for identifying the first $\lambda$ good arms and obtaining the explicit form of $T$, and lastly proving the sample complexity for identifying all arms.

\subsection{Correctness of lil'HDoC}

\begin{theorem}
\label{thm:correctness}
    With probability at least $1 - \delta$, lil'HDoC correctly identifies every arms.
\end{theorem}
\begin{proof}[Proof of Theorem \ref{thm:correctness}]
Let $D$ as an event that $\forall \: i \in [K]$ and $t \geq 1$,
    \begin{equation}
        | \hat{\mu}_{i, t} - \mu_i | \leq U(t, \frac{\delta}{c_\epsilon K})
    \end{equation}
    By Lemma \ref{lma:iteratedbound} and the union bound, we know that event $D$ happen with probability at least
    \begin{align}
        1 - K c_{\epsilon} (\dfrac{\delta}{c_\epsilon K})^{1 + \epsilon}
        \geq\: 1 - K c_{\epsilon} (\dfrac{\delta}{c_\epsilon K})
        \geq\: 1 - \delta
        \label{eq:D_prob}
    \end{align}
    Inequality \ref{eq:D_prob} is because $\epsilon \in [0, 1]$ and $c_\epsilon \geq 1$.
    So event D happens with probability at least $1 - \delta$.
    Therefore, if\begin{equation}
        \hat{\mu}_{i, T_i(t)} - U(T_i(t), \dfrac{\delta}{c_\epsilon K}) \geq \xi
    \end{equation} holds under event $D$,\begin{align}
        & \hat{\mu}_{i, T_i(t)} - U(T_i(t), \dfrac{\delta}{c_\epsilon K}) \geq \xi \\
        \Rightarrow\: & \hat{\mu}_{i, T_i(t)} - | \hat{\mu}_{i, T_i(t)} - \mu_i | \geq \xi\\
        \Rightarrow\: & \mu_i \geq \xi
    \end{align}
    and if\begin{equation}
        \hat{\mu}_{i, T_i(t)} + U(T_i(t), \dfrac{\delta}{c_\epsilon K}) < \xi
    \end{equation} holds under event $D$,\begin{align}
        & \hat{\mu}_{i, T_i(t)} + U(T_i(t), \dfrac{\delta}{c_\epsilon K}) < \xi \\
        \Rightarrow\: & \hat{\mu}_{i, T_i(t)} + | \hat{\mu}_{i, T_i(t)} - \mu_i | < \xi\\
        \Rightarrow\: & \mu_i < \xi
    \end{align}
    So our algorithm output the correct answer when event $D$ holds, so the error rate of our algorithm is at most $\delta$.
\end{proof}

\subsection{First $\lambda$ Arms Sampling Complexity}

\begin{theorem}[First $\lambda$ Arms Sample Complexity]
\label{thm:lambdacmoplexity}
    After conducting at most $O\big(\log(K + 1) \log(\max(\frac{1}{\delta}, e))\big)$ samples on each arm, our confidence bound will be less than that of HDoC.
\end{theorem}

\begin{proof}
    See Appendix 2.
\end{proof}

\subsection{Total Sample Complexity} \label{sec:totalcomplexity}

\begin{theorem}[Sample Complexity]
    Let T = 1, then with probability at least $1 - \delta$, lil'HDoC identifies arm $i$ with at most 
    \begin{align*}
        \dfrac{2 (1 + \epsilon) (1 + \sqrt{\epsilon})^2}{\Delta_i^2} \log{\bigg[\dfrac{2 c_\epsilon K\log{\big[ \dfrac{2 c_\epsilon K (1 + \sqrt{\epsilon})^2(1 + \epsilon)^2}{\delta \Delta_i^2}\big] }}{\delta}\bigg]}
    \end{align*}
    times of sampling.
\end{theorem}

\begin{proof}
    See Appendix 3.
\end{proof}

\section{Experiment} \label{sec:experiment}
The goal of our experiments is to learn (1) lil'HDoC improves the sample complexity of identifying first $\lambda$ good arms in practice, (2) lil'HDoC improves the sample complexity of identifying all arms in practice. In the following experiment, we will focus on the datasets given the challenging situation.

\subsection{Dataset}
\textbf{Syntactic Dataset}
One syntactic dataset is provided. It has six arms with expected reward $0.007$, $0.006$, $0.005$, $0.003$, $0.002$, $0.001$ respectively, and the threshold is $0.004$. We conduct the experiment over 10 independent runs.\\
\textbf{Real World Dataset}
We generate the experiment data from three real-world datasets: Covertype ~\cite{asuncion2007uci}, Jester ~\cite{goldberg2001eigentaste}, and MovieLens \cite{harper2015movielens}. We conduct the experiment over 25 independent runs.
The Covertype dataset classifies the cover type of northern Colorado forest areas in 7 classes.
For Covertype, we use the method similar to \cite{dudik2011doubly,da2022fast,tsai2023differential} to transform multi--class dataset to bandit dataset, and we divide the mean by 10 to make the dataset more challenging. The threshold is set as the arithmetic mean of the reward of the 3rd best arm and the 4th best arm. We conduct the experiment over 25 independent runs.
Jester dataset provides continuous ratings in [-10, 10] for 100 jokes from 73421 users.
For Jester, we create a recommendation system bandit problem as follows. We first count the average rating of 100 jokes and divide then by 10 in order to increase difficulty, and scale the rating from $[-10, 10]$ to $[0, 1]$. The threshold is set as the arithmetic mean of the reward of the 25-th best arm and the 26-th best arm. We conduct the experiment over 25 independent runs.
MovieLens dataset provides 100,000 ratings ranging from [0, 5], from 1000 users on 1682 movies.
For MovieLens, we first average the rating of each movie, and divide the average ratings by 100. The threshold is set as the arithmetic mean of the reward of the 168-th best arm and the 169-th best arm.
Here we provide the sample complexity of identifying the first 25 good movies. We conduct the experiment over 20 independent runs.

\subsection{Baseline}
We choose HDoC and LUCB--G as the baseline since they are the top-2 models compared here \cite{kano2019good}.
All algorithms including ours consist of two stages, sampling and identifying. The former decides which arm to pull in the next round, and the latter decides whether the pulled arm can be identified as a good or bad arm.

\noindent\textbf{Sampling Method}
\begin{enumerate}
    \item HDoC : Pull arm $\hat{a}^* = \text{argmax}_{i\in A} \overline{\mu}_{i,t}$~, where
    $\overline{\mu}_{i,t} = \hat{\mu}_{i,t} + \sqrt{\frac{\log(t)}{2N_i(t)}}$
    \item LUCB-G : Pull arm $\hat{a}^* = \text{argmax}_{i\in A} \overline{\mu}_{i,t}$~, where
    $\overline{\mu}_{i,t} = \hat{\mu}_{i,t} + \sqrt{\frac{\log{4KN_i^2(t)}/\delta}{2N_i(t)}}$
    \item lil'HDoC : Pull arm $\hat{a}^* = \text{argmax}_{i \in A} \overline{\mu}_{i,t}$, where
    $\overline{\mu}_{i,t} = \hat{\mu}_{i,t} + \sqrt{\frac{\log{t}}{2N_i(t)}}$
\end{enumerate}

\noindent\textbf{Identifying Method}
Here we show the confidence bound of algorithms
\begin{enumerate}
    \item HDoC : $\sqrt{\frac{\log{4KN_i^2(t)/\delta}}{2N_i(t)}}$
    \item LUCB-G : $\sqrt{\frac{\log{4KN_i^2(t)/\delta}}{2N_i(t)}}$
    \item lil'HDoC : $(1 + \sqrt{\epsilon}) \sqrt{\frac{(1 + \epsilon)}{2N_i(t)} \log \frac{c_\epsilon K \log((1 + \epsilon) N_i(t))}{\delta}}$
\end{enumerate}

\begin{table}[t!]
    \begin{subtable}[h]{1.0\textwidth}
        \centering
        \begin{tabular}{ |c||c|c|c|}
        \hline
         & HDoC & LUCB--G & lil'HDoC\\
         \hline
        $\tau_1$ & 3.92 \textpm\ 0.24 & {\bf 1.98 \textpm\ 0.06} & 2.67 \textpm\ 0.18\\
        $\tau_2$ & 9.30 \textpm\ 0.21 & 22.91 \textpm\ 0.48 & {\bf 6.00 \textpm\ 0.24}\\
        $\tau_3$ & 30.10 \textpm\ 0.78 & 55.56 \textpm\ 1.01 & {\bf 17.60 \textpm\ 0.51}\\
        $\tau_{stop}$ & 551.13 \textpm\ 10.89 & 555.58 \textpm\ 10.12 & {\bf 315.85 \textpm\ 6.49}\\
        \hline
        \end{tabular}
        \caption{Average and standard deviation of sampling times over synthetic dataset. All number is divided by 100000.}
        \label{tbl:syntheticdata}
    \end{subtable}
    \hfill
    \begin{subtable}[h]{1.0\textwidth}
        \centering
        \begin{tabular}{ |c||c|c|c|}
        \hline
         & HDoC & LUCB--G & lil'HDoC\\
         \hline
        $\tau_1$ & {\bf 0.14 \textpm\ 0.02} & {\bf 0.13 \textpm\ 0.02} & {\bf 0.15 \textpm\ 0.01}\\
        $\tau_2$ & {\bf 0.29 \textpm\ 0.03} & 104.64 \textpm\ 3.50 & {\bf 0.25 \textpm\ 0.02}\\
        $\tau_3$ & 137.55 \textpm\ 3.24 & 317.71 \textpm\ 0.71 & {\bf 83.26 \textpm\ 2.13}\\
        $\tau_{stop}$ & 3225.9 \textpm\ 53.1 & 3177.1 \textpm\ 71.4 & {\bf 1860.3 \textpm\ 38.2}\\
        \hline
        \end{tabular}
        \caption{Average and standard deviation of sample times over Covertype dataset. All numbers in this table is divided by 10000.}
        \label{tbl:covertypedata}
     \end{subtable}
\hfill
    \begin{subtable}[h]{1.0\textwidth}
        \centering
        \begin{tabular}{ |c||c|c|c|}
        \hline
         & HDoC & LUCB--G & lil'HDoC\\
         \hline
        $\tau_3$ & {\bf 1.61 \textpm\ 0.23} & {\bf 5.92 \textpm\ 64.6} & {\bf 1.42 \textpm\ 0.24}\\
        $\tau_6$ & {\bf 2.69 \textpm\ 0.30} & 95.79 \textpm\ 72.5 & {\bf 2.35 \textpm\ 0.25}\\
        $\tau_9$ & {\bf 4.18 \textpm\ 0.33} & 217.78 \textpm\ 93.3 & {\bf 3.57 \textpm\ 0.30}\\
        $\tau_{12}$ & 6.26 \textpm\ 0.42 & 275.79 \textpm\ 55.1 & {\bf 4.85 \textpm\ 0.49}\\
        $\tau_{15}$ & 8.31 \textpm\ 0.50 & 301.77 \textpm\ 32.4 & {\bf 6.55 \textpm\ 0.50}\\
        $\tau_{18}$ & 18.30 \textpm\ 1.97 & 320.13 \textpm\ 55.9 & {\bf 13.94 \textpm\ 1.18}\\
        $\tau_{21}$ & 51.87 \textpm\ 4.47 & 478.79 \textpm\ 79.2 & {\bf 36.18 \textpm\ 2.92}\\
        $\tau_{24}$ & 95.10 \textpm\ 4.75 & 551.42 \textpm\ 36.2 & {\bf 63.47 \textpm\ 4.24}\\
        $\tau_{stop}$ & 563.63 \textpm\ 21.1 & 560.40 \textpm\ 31.6 & {\bf 324.70 \textpm\ 19.5}\\
        \hline
        \end{tabular}
        \caption{Average and standard deviation of sampling times over Jester dataset. All number is divided by 100000.}
        \label{tbl:jesterdata}
    \end{subtable}
    \hfill
    \begin{subtable}[h]{1.0\textwidth}
        \centering
        \begin{tabular}{|c||c|c|c|}
        \hline
         & HDoC & LUCB--G & lil'HDoC\\
         \hline
        $\tau_3$ & 27.02 \textpm\ 0.83 & NA & {\bf 24.39 \textpm\ 0.70}\\
        $\tau_6$ & 29.16 \textpm\ 0.93 & NA & {\bf 26.07 \textpm\ 0.90}\\
        $\tau_9$ & 31.38 \textpm\ 1.11 & NA & {\bf 28.47 \textpm\ 1.01}\\
        $\tau_{12}$ & 85.36 \textpm\ 4.17 & NA & {\bf 76.49 \textpm\ 4.56}\\
        $\tau_{15}$ & 98.47 \textpm\ 3.19 & NA & {\bf 87.14 \textpm\ 3.85}\\
        $\tau_{18}$ & 107.14 \textpm\ 4.31 & NA & {\bf 94.71 \textpm\ 3.09}\\
        $\tau_{21}$ & 123.67 \textpm\ 7.59 & NA & {\bf 106.60 \textpm\ 3.61}\\
        $\tau_{24}$ & {\bf 162.94 \textpm\ 11.58} & NA & {\bf 146.25 \textpm\ 5.88}\\
        \hline
        \end{tabular}
        \caption{Average and standard deviation of sampling times over MovieLens dataset. All number is divided by 1000000. NA in the entries means that the sample times is more than 1.5e+8}
        \label{tbl:movielendata}
    \end{subtable}

    \caption{Experiment results of 4 dataset. Best results are in bold fonts.}
\end{table}

\begin{figure}[t]
\centering
\begin{subfigure}{.45\textwidth}
  \centering
  \includegraphics[width=\linewidth]{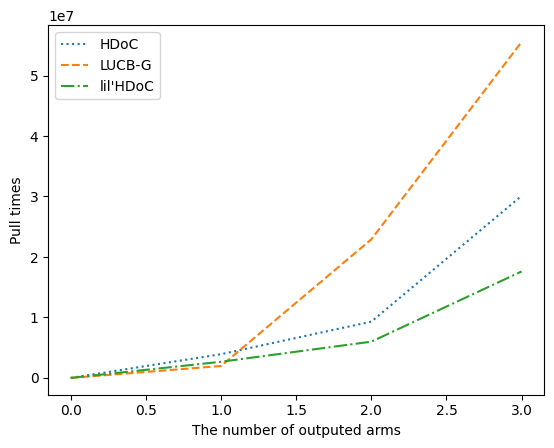}
  \caption{Synthetic dataset}
  \label{fig:synthetic}
\end{subfigure}
\begin{subfigure}{.45\textwidth}
  \centering
  \includegraphics[width=\linewidth]{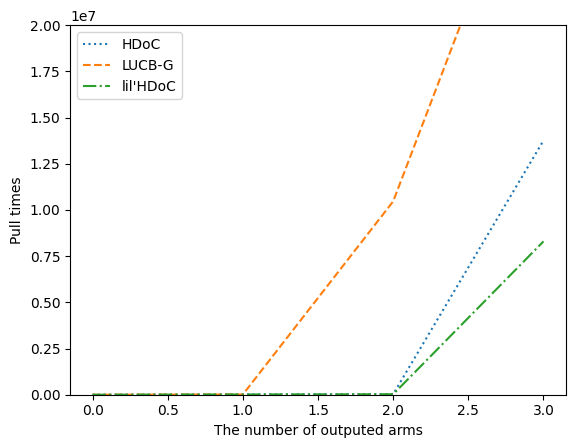}
  \caption{Covertype dataset}
  \label{fig:covertype}
\end{subfigure}
\centering
\begin{subfigure}{.45\textwidth}
  \centering
  \includegraphics[width=\linewidth]{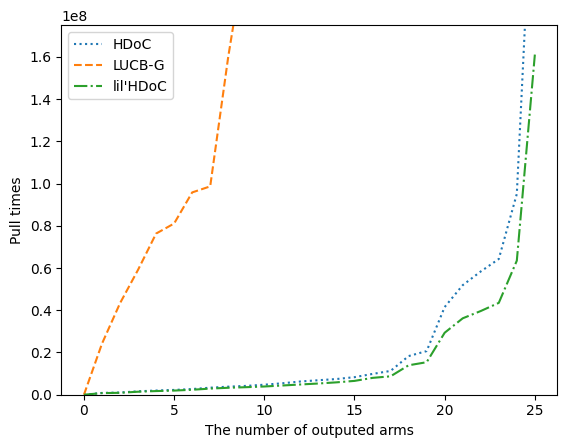}
  \caption{Jester dataset}
  \label{fig:jester}
\end{subfigure}
\begin{subfigure}{.45\textwidth}
  \centering
  \includegraphics[width=\linewidth]{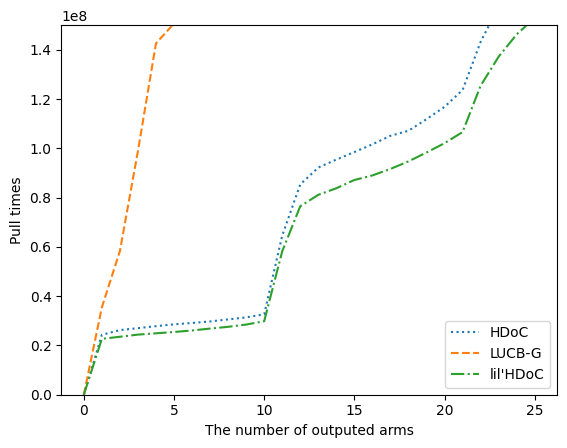}
  \caption{MovieLens dataset}
  \label{fig:movielen}
\end{subfigure}
\caption{Experimental Results}
\end{figure}

\subsection{Results}
The experimental results are listed from Figure \ref{fig:synthetic} to Figure \ref{fig:movielen}. In these figures, the x axis indicates the number of identified good arms, and the y axis indicate the number of sample times. 
The detailed information of Figure \ref{fig:synthetic} to Figure \ref{fig:movielen} is in Table \ref{tbl:syntheticdata} to Table \ref{tbl:movielendata}. In addition, we only list the $\tau_{stop}$ in the table since it is often very large with respect to $\tau_\lambda$.

In every experiments, lil'HDoC outperforms HDoC, especially when an arm required more sample times to identify it. We can see this from Table \ref{tbl:syntheticdata}, Table \ref{tbl:covertypedata}, Table \ref{tbl:jesterdata}, and Table \ref{tbl:movielendata}.
From the last row of the three table, we can also see that the sampling times of identifying all arms in lil'HDoC outperforms the sampling times of identifying all arms in HDoC and LUCB--G.
The effect of lil'HDoC is not obvious when the required sample times is small. It is because that although the rate of convergence is faster in the confidence bound of lil'HDoC, the constant in the confidence bound of lil'HDoC is larger than that of HDoC, so when the sample times is not large, the effect will be diluted.
The other issue is that in Jester and MovieLens, LUCB--G may identify many bad arms before all good arms is identified.
This is because the sampling method of LUCB--G doesn't put the total sampling times $t$ into consideration, so its exploration part would be less than HDoC and lil'HDoC. Thus, LUCB--G will exploit the suboptimal arms.

\section{Conclusion} \label{sec:conclusion}
In this paper, we propose a new algorithm lil'HDoC, based on the HDoC algorithm in GAI \cite{kano2019good}.
Intuitively, we leverage the fact that under challenging situation (when the threshold gap is small $<$ 0.01) every arms require a huge number of sample times.
Thus, we can sample each arm for a rather small number of times in the beginning to obtain more confidence on the goodness/badness of arms, which can lead to a tighter confidence bound in the identifying method.
From the theoretical perspective, the first $\lambda$ good arms of lil'HDoC is bounded by the original HDoC algorithm except for one negligible term under challenging situation, and the total sample complexity of lil'HDoC is less than HDoC by decrease the $\frac{1}{\Delta} \log \frac{1}{\Delta}$ term to $\frac{1}{\Delta} \log \log \frac{1}{\Delta}$, which makes a conspicuous improvement when $\Delta$ is small.
From the practical perspective, on both synthetic and read world dataset, lil'HDoC outperforms HDoC under the same acceptance rate.
Therefore, we conclude that lil'HDoC outperform HDoC and LUCB--G from theoretical and empirical performance under challenging situation.
%
%
%
\bibliographystyle{splncs04}
\bibliography{samplepaper}

\end{document}


\begin{center}
  \Large\bfseries\boldmath
  {lil’HDoC: An Algorithm for Good Arm Identification under Small Threshold Gap\\ \large Supplementary material}
\end{center}

\appendix
\setcounter{equation}{6}

\section{Correctness}
\begin{theorem}
\label{thm:correctness}
    With probability at least $1 - \delta$, lil'HDoC correctly identifies every arms.
\end{theorem}
\begin{proof}[Proof of Theorem \ref{thm:correctness}]
Let $D$ as an event that $\forall \: i \in [K]$ and $t \geq 1$,
    \begin{equation}
        | \hat{\mu}_{i, t} - \mu_i | \leq U(t, \frac{\delta}{c_\epsilon K})
    \end{equation}
    By Lemma \ref{main:lma:iteratedbound} and the union bound, we know that event $D$ happen with probability at least
    \begin{align}
        1 - K c_{\epsilon} (\dfrac{\delta}{c_\epsilon K})^{1 + \epsilon}
        \geq\: 1 - K c_{\epsilon} (\dfrac{\delta}{c_\epsilon K})
        \geq\: 1 - \delta
        \label{eq:D_prob}
    \end{align}
    Inequality \ref{eq:D_prob} is because $\epsilon \in [0, 1]$ and $c_\epsilon \geq 1$.
    So event D happens with probability at least $1 - \delta$.
    Therefore, if\begin{equation}
        \hat{\mu}_{i, T_i(t)} - U(T_i(t), \dfrac{\delta}{c_\epsilon K}) \geq \xi
    \end{equation} holds under event $D$,\begin{align}
        & \hat{\mu}_{i, T_i(t)} - U(T_i(t), \dfrac{\delta}{c_\epsilon K}) \geq \xi \\
        \Rightarrow\: & \hat{\mu}_{i, T_i(t)} - | \hat{\mu}_{i, T_i(t)} - \mu_i | \geq \xi\\
        \Rightarrow\: & \mu_i \geq \xi
    \end{align}
    and if\begin{equation}
        \hat{\mu}_{i, T_i(t)} + U(T_i(t), \dfrac{\delta}{c_\epsilon K}) < \xi
    \end{equation} holds under event $D$,\begin{align}
        & \hat{\mu}_{i, T_i(t)} + U(T_i(t), \dfrac{\delta}{c_\epsilon K}) < \xi \\
        \Rightarrow\: & \hat{\mu}_{i, T_i(t)} + | \hat{\mu}_{i, T_i(t)} - \mu_i | < \xi\\
        \Rightarrow\: & \mu_i < \xi
    \end{align}
    So our algorithm output the correct answer when event $D$ holds, so the error rate of our algorithm is at most $\delta$.
\end{proof}

\section{First $\lambda$ Arms Sampling Complexity}
Now we prove the sampling complexity of identifying the first $\lambda$ good arms.
We compare our confidence bound $U(t, \frac{\rho}{c_\epsilon K})$ in the identifying method with the confidence bound in the identifying method of HDoC. We have found that there exists a value of $T$ such that for all $N_i(t) \geq T$, the confidence bound of our algorithm is smaller than that of HDoC. Therefore, once the HDoC algorithm identifies an arm, our algorithm, lil'HDoC, is also able to identify it. We can determine the value of $T$ as the smallest integer such that
\begin{align*}
    &\: \sqrt{\frac{\log \frac{4KT^2}{\delta}}{2T}} \geq (1 + \sqrt{\epsilon}) \sqrt{\frac{(1 + \epsilon)}{2T} \log \frac{K c_\epsilon \log [(1 + \epsilon) T]}{\delta}}\\
    \Rightarrow &\: \log (\frac{4KT^2}{\delta}) \geq (1 + \epsilon)^2(1 + \epsilon) \log \frac{K c_\epsilon \log [(1 + \epsilon) T]}{\delta}\\
    \Rightarrow &\: \frac{4KT^2}{\delta} \geq \big(K \frac{c_\epsilon}{\delta} \big[\log [(1 + \epsilon)T] \big] \big)^{(1 + \sqrt{\epsilon})^2(1 + \epsilon)}\\
\end{align*}
    Let $(1 + \sqrt{\epsilon})^2(1 + \epsilon) = r$, then the above inequality is equal to
\begin{equation}
    \frac{T^2}{\big[\log [(1 + \epsilon)T] \big]^{r}} \geq \frac{1}{4} K^{r - 1} (\frac{1}{\delta})^{r - 1} (c_\epsilon)^{r}
\end{equation}
Let $B = (K + 1),\: C = \max(\frac{1}{\delta}, e)$, if
\begin{equation}
    r - 1 \leq min(\frac{\log \log B}{\log B}, \frac{\log \log C}{\log C}) 
\end{equation}
Also $\forall\: x \geq e,\: x \in \mathbb{R}$, \begin{equation} \label{ieq:loglogbound}
    \frac{\log \log x}{\log x} \leq \frac{1}{2}
\end{equation}
, then
\begin{align}
    \frac{T^2}{\big[\log [(1 + \epsilon)T] \big]^{r}} \geq \frac{1}{4} K^{r - 1} (\frac{1}{\delta})^{r - 1} (c_\epsilon)^{r} \label{ieq:beforeloose}
\end{align}
\begin{align}
    \Rightarrow &\: \frac{T^2}{\big[\log [(1 + \epsilon)T] \big]^{r}} \geq \frac{1}{4} (K + 1)^{r - 1} (\max(\frac{1}{\delta}, e))^{r - 1} (c_\epsilon)^{r}
    \label{ieq:firstloose}
\end{align}
\begin{align}
    \Rightarrow &\: \frac{T^2}{\big[\log [(1 + \epsilon)T] \big]^{r}} \geq \frac{1}{4} \log (K + 1) \log (\max(\frac{1}{\delta}, e)) (c_\epsilon)^{1.5}
    \label{ieq:secondloose}
\end{align}
Inequality (\ref{ieq:beforeloose}) to (\ref{ieq:firstloose}) holds true since L.H.S. of (\ref{ieq:beforeloose}) is monotonically increasing for $T \in \mathbb{N}$ and $r \leq 1.5$, and R.H.S of (\ref{ieq:firstloose}) is greater than R.H.S of (\ref{ieq:beforeloose}). Consequently, the minimum value of $T$ that satisfies (\ref{ieq:firstloose}) also satisfies (\ref{ieq:beforeloose}).
Furthermore, for all $x \in \mathbb{N}$, we can observe that:
\begin{equation}
\label{ieq:logtvst}
[\log[(1 + \epsilon)x]]^r \leq [\log[rx]]^r \leq [\log[1.5x]]^{1.5} \leq x
\end{equation}
By combining (\ref{ieq:secondloose}) with (\ref{ieq:logtvst}), we get:
\begin{equation}
T \geq \frac{1}{4} \log (K + 1) \log (\max(\frac{1}{\delta}, e)) (c_\epsilon)^{1.5}
\end{equation}
Therefore, after conducting at most $O\big(\log(K + 1) \log(\max(\frac{1}{\delta}, e))\big)$ samples on each arm, our confidence bound will be less than that of HDoC.

\section{Total Sample Complexity} \label{sec:totalcomplexity}
In this section, we provide an upper bound on the total sample complexity of lil'HDoC. The following theorem states the sample complexity:
\begin{theorem}[Sample Complexity]
\label{thm:samplecomplexity}
    Let T = 1, then with probability at least $1 - \delta$, lil'HDoC identifies arm $i$ with at most 
    \begin{align}
        \dfrac{2 (1 + \epsilon) (1 + \sqrt{\epsilon})^2}{\Delta_i^2} \log{\bigg[\dfrac{2 c_\epsilon K\log{\big[ \dfrac{2 c_\epsilon K (1 + \sqrt{\epsilon})^2(1 + \epsilon)^2}{\delta \Delta_i^2}\big] }}{\delta}\bigg]}
    \end{align}
    times of sampling.
    \label{eq:totalcomplexity}
\end{theorem}

To prove Theorem \ref{thm:samplecomplexity}, we follow a similar approach to that of lil'RandLUCB \cite{jiang2017practical}, which deals with the top--m multi--armed bandit problem. We leverage the deterministic property of the given threshold, which allows us to bypass tangled parts in the proof of lil'RandLUCB algorithm. In contrast, lil'RandLUCB has to handle the relationship between the $m^{th}$ largest and the $(m + 1)^{th}$ largest arm, which is not deterministic.

\begin{proof} [Proof of Theorem  \ref{thm:samplecomplexity}]
    Let event $D$ holds all over this proof.
    For the good arm $i$, the minimum number of sampling $\tau_i$ is the smallest integer $t$ satisfying
    \begin{align}
        & \hat{\mu}_{i, t} - \xi \geq U(t, \frac{\delta}{c_\epsilon K})\\
        \Rightarrow & \hat{\mu}_{i, t} - \mu_i + \mu_i - \xi \geq U(t, \frac{\delta}{c_\epsilon K})\\
        \Rightarrow & \hat{\mu}_{i, t} - \mu_i + \Delta_i \geq U(t, \frac{\delta}{c_\epsilon K})\\
        \Rightarrow & \dfrac{\Delta_i}{2} \geq U(t, \frac{\delta}{c_\epsilon K}) \label{ieq:goodarmbound}
    \end{align}
    Inequality \ref{ieq:goodarmbound} is because event $D$ make sure $| \hat{\mu}_{i, t} - \mu_i | \leq U(t, \frac{\delta}{c_\epsilon K})$, and the smaller LHS implies the larger $t$, so we let $\mu_i - \hat{\mu}_{i, t} = -U(t, \frac{\delta}{c_\epsilon K})$.
    
    For the bad arm $i$, the minimum number of sampling $\tau_i$ is the smallest integer $t$ satisfying
    \begin{align}
        & \xi - \hat{\mu}_{i, t} \geq U(t, \frac{\delta}{c_\epsilon K})\\
        \Rightarrow & \xi  - \mu_i + \mu_i - \hat{\mu}_{i, t} \geq U(t, \frac{\delta}{c_\epsilon K})\\
        \Rightarrow & \Delta_i + \mu_i - \hat{\mu}_{i, t} \geq U(t, \frac{\delta}{c_\epsilon K})\\
        \Rightarrow & \dfrac{\Delta_i}{2} \geq U(t, \frac{\delta}{c_\epsilon K}) \label{ieq:badarmbound}
    \end{align}
    Inequality \ref{ieq:badarmbound} is because event $D$ make sure $| \hat{\mu}_{i, t} - \mu_i | \leq U(t, \frac{\delta}{c_\epsilon K})$, and the smaller LHS implies the larger $t$, so we let $\mu_i - \hat{\mu}_{i, t} = -U(t, \frac{\delta}{c_\epsilon K})$. This result coincide with the good arm case.
    Apply Lemma \ref{main:lma:iteratedlogtransform} to (\ref{ieq:goodarmbound}) and  (\ref{ieq:badarmbound}), we can find that if
    \begin{align}
        t > \dfrac{8 \sigma^2 (1 + \epsilon) (1 + \sqrt{\epsilon})^2}{\Delta_i^2}  \log{\bigg[\dfrac{2 c _\epsilon K\log{\big[ \dfrac{8 c_\epsilon K\sigma^2(1 + \sqrt{\epsilon})^2(1 + \epsilon)^2}{\delta \Delta_i^2}\big] }}{\delta}\bigg]}
    \end{align}
    under event $D$, then it can identify whether arm $i$ is good or bad.
    Eventually, Bernoulli distribution is $\frac{1}{2}$--sub--gaussian, and event $D$ holds with probability $1 - \delta$, so with probability $1 - \delta$, lil'HDoC identifies arm i with sample complexity

    \begin{align} \label{ieq:totalcomplexity}
        t > \dfrac{2 (1 + \epsilon) (1 + \sqrt{\epsilon})^2}{\Delta_i^2} \log{\bigg[\dfrac{2 c _\epsilon K\log{\big[ \dfrac{2 c_\epsilon K (1 + \sqrt{\epsilon})^2(1 + \epsilon)^2}{\delta \Delta_i^2}\big] }}{\delta}\bigg]}
    \end{align}
\end{proof}

$T$ is often larger than 1, and with $\Delta = \min_{i \in [K]} \Delta_i$ we can deriev the following:

\begin{corollary}
    With probability $1 - \delta$, the total sample complexity of lil'HDoC when $T$ is larger than the R.H.S of inequality \ref{ieq:totalcomplexity} for all arms.
    \begin{equation}
        O\bigg( \dfrac{K \log \frac{1}{\delta} + K \log K + K \log \log \frac{1}{\Delta}}{\Delta^2} \bigg)
    \end{equation}
    \begin{equation}
    \label{eq:uniformsample}
        O\bigg( K \log (K + 1) \log (\max(\frac{1}{\delta}, e)) \bigg)
    \end{equation}
    \ref{ieq:totalcomplexity} for all arms.
\end{corollary}
Equation~(\ref{eq:uniformsample}) holds because from line 8 to line 11 in Algorithm \ref{main:alg:lilhdoc}, lil'HDoC doesn't identify arms.
An obvious corollary is followed,
\begin{corollary}
    From line 8 to line 11 in Algorithm \ref{main:alg:lilhdoc}, if lil'HDoC performs identification once an arm is sampled, the total sample complexity becomes
    \begin{equation}
        O\bigg( \dfrac{K \log \frac{1}{\delta} + K \log K + K \log \log \frac{1}{\Delta}}{\Delta^2} \bigg)
    \end{equation}
\end{corollary}

\bibliographystyle{abbrvnat} 
\bibliography{sample}